\documentclass{MITcsail}

\usepackage[utf8]{inputenc} 
\usepackage[T1]{fontenc}    
\usepackage[colorlinks=true, citecolor=blue, linkcolor=black]{hyperref}     
\usepackage{url}            
\usepackage{booktabs}       
\usepackage{amsfonts}       
\usepackage{nicefrac}       
\usepackage{microtype}      
\usepackage{xcolor}         

\usepackage{graphicx} 
\usepackage{amsmath,amsthm,amssymb}
\usepackage{ifthen}
\usepackage{todonotes}
\usepackage{comment}
\usepackage{multirow}
\usepackage[inline]{enumitem}

\usepackage{cleveref}

\newcommand{\REAL}{\mathbb{R}}
\newcommand{\term}[1]{\left( #1\right)}
\newcommand{\abs}[1]{\left| #1 \right|}
\newcommand{\br}[1]{\left\lbrace #1 \right\rbrace}
\newtheorem{assumption}[theorem]{Assumption}
\newtheorem{claim}[theorem]{Claim}
\DeclareMathOperator*{\argmin}{arg\,min}

\newcommand{\XY}[2]{
\ifthenelse{\equal{#2}{}}{\mathbf{X}, \mathbf{Y}}{\mathbf{#1}, \mathbf{#2}}
}
\renewcommand{\Tilde}[1]{\widetilde{#1}}

\newcommand{\RFFX}{\Tilde{\mathbf{X}}}
\newcommand{\y}[1]{y_{\mathbf{#1}}}
\newcommand{\rahimi}[1]{\Tilde{\mathbf{#1}}}
\newcommand{\X}{\mathbf{X}}
\newcommand{\K}{\mathbf{K}}
\newcommand{\bterm}[1]{\left[ #1 \right]}
\newcommand{\alphak}{\alpha^{\lambda}_{\bterm{\X, y, k}}}
\newcommand{\fk}{f^\lambda_{\bterm{\X, y, k}}}
\newcommand{\fkt}{f^\lambda_{\bterm{\RFFX, y, \phi}}}
\newcommand{\BS}{\mathbf{S}}
\newcommand{\fks}{f^{\lambda,\X,y}_{\bterm{\BS, y_{\BS}, k}}}
\newcommand{\I}{\mathbf{I}}
\newcommand{\A}{\mathbf{A}}
\newcommand{\maxterm}[2]{\max\br{#1, #2}}
\newcommand{\minterm}[2]{\min\br{#1, #2}}

\title{On the Size and Approximation Error of Distilled Sets}

\author
{\href{https://scholar.google.com/citations?user=6r72e-MAAAAJ&hl=en}{\color{blue}Alaa Maalouf~$^{1,\S,}$\footnote{Correspondence E-mail: alaam@mit.edu. $^\S$Equal contribution.}}, \href{https://scholar.google.com/citations?user=721xaz0AAAAJ&hl=en}{\color{blue} Murad Tukan~$^{2,\S}$}, 
 \href{https://scholar.google.com/citations?user=vokGv-gAAAAJ&hl=en} {\color{blue} Noel Loo~$^{1}$}, \href{https://scholar.google.com/citations?user=YarJF3QAAAAJ&hl=en}{\color{blue} Ramin Hasani~$^{1}$,} \href{https://scholar.google.com/citations?user=fIupeSAAAAAJ&hl=en}{ \color{blue}Mathias Lechner~$^{1}$}, \href{https://scholar.google.com/citations?user=910z20QAAAAJ&hl=en}{\color{blue} Daniela Rus~$^{1}$}\\
\vspace{1em} 
\color{magenta}{ $^{1}$Computer Science and Artificial Intelligence Lab (CSAIL),
Massachusetts Institute of Technology (MIT)}\\
{ $^{2}$DataHeroes} \vspace{2em}
}

\begin{document}

\maketitle
\thispagestyle{firstpagestyle}

\begin{abstract}
  Dataset Distillation is the task of synthesizing small datasets from large ones while still retaining comparable predictive accuracy to the original uncompressed dataset. Despite significant empirical progress in recent years, there is little understanding of the theoretical limitations/guarantees of dataset distillation, specifically, what excess risk is achieved by distillation compared to the original dataset, and how large are distilled datasets? In this work, we take a theoretical view on kernel ridge regression (KRR) based methods of dataset distillation such as Kernel Inducing Points. By transforming ridge regression in random Fourier features (RFF) space, we provide the first proof of the existence of small (size) distilled datasets and their corresponding excess risk for shift-invariant kernels. We prove that a small set of instances exists in the original input space such that its solution in the RFF space coincides with the solution of the original data. We further show that a KRR solution can be generated using this distilled set of instances which gives an approximation towards the KRR solution optimized on the full input data. The size of this set is linear in the dimension of the RFF space of the input set or alternatively near linear in the number of effective degrees of freedom, which is a function of the kernel, number of datapoints, and the regularization parameter $\lambda$. The error bound of this distilled set is also a function of $\lambda$.  
  We verify our bounds analytically and empirically.
\end{abstract}

\section{Introduction}
Motivated by the growing data demands of modern deep learning, dataset distillation~\cite{zhao2021DSA,KIP2,frepo,wang2018dataset} aims to summarize large datasets into significantly smaller synthetic \textit{distilled} datasets, which when trained on retain high predictive accuracy, comparable to the original dataset. These distilled datasets have applications in continual learning \citep{frepo, continual_dd}, architecture search \citep{NAS_with_DD}, and privacy preservation \citep{chen2022privateset}. Recent years have seen the development of numerous distillation algorithms, but despite this progress, the field has remained largely empirical. Specifically, there is little understanding of what makes one dataset ``easier to distill" than another, or whether such small synthetic datasets even exist.

This work aims to fill this gap by providing the first bounds on the sufficient size and relative error associated with distilled datasets. Noting prior work relating neural network training to kernel ridge regression (KRR), we consider dataset distillation in the kernel ridge regression settings with shift-invariant kernels. By casting the problem into the Random Fourier Feature (RFF) space, we show that:

\textbf{The size and relative error of distilled datasets is governed by the kernel's ``number of effective degrees of freedom", $d^\lambda_k$}. Specifically, in Section~\ref{sec:main}, we show that distilled sets of size $\Omega(d^\lambda_k \log d^\lambda_k)$, exist, with $12\lambda + 2 \mathcal{L}_\lambda$ predictive error on the training dataset, and only $8\lambda$ error with respect to the optimal solution computed on the full dataset, where $\lambda$ is the kernel ridge regression regularization parameter and $\mathcal{L}_\lambda$ the KRR training error on the original dataset; see Theorem~\ref{thm:existence} and Remark~\ref{rem:main} for full details.

\textbf{These bounds hold in practice for both real and synthetic datasets}. In \cref{sec:experiments}, we validate our theorem by distilling synthetic and real datasets with varying sizes and values of $d^\lambda_k$, showing that in all scenarios our bounds accurately predict the error associated with distillation.

\section{Related work}
\paragraph{Coresets. }
Coresets are weighted selections from a larger training dataset, which, when used for training, yield similar outcomes as if the whole dataset was used \cite{munteanu2018coresets, MirzasoleimanBL20,maalouf2019fast,jubran2019introduction,maalouf2022unified}. The key benefit of using coresets is that they significantly speed up the training process, unlike when the full data set is used.
Current methods for picking out coresets incorporate clustering techniques \cite{feldman2011unified,jubran2020sets,lucic2016strong, BachemLHK16,maalouf2021average}, bilevel optimization \cite{borsos2020coresets}, sensitivity analysis \cite{munteanu2018coresets, HugginsCB16, TukanMF20, maalouf2020tight,tukan2021svms,tukan2022projective,maalouf2022sine}, and surrogate models for approximation \cite{tukan2023provable}.
Newer strategies are specifically designed for neural networks, where before each training epoch, coresets are chosen such that their gradients align with the gradients of the entire dataset \cite{MirzasoleimanBL20, PooladzandiDM22,tukan2023provable}, followed by training the model on the chosen coreset. Although coresets are usually theoretically supported,  these methods fall short when the aim is to compute a coreset once for a full training procedure.

\paragraph{Dataset Distillation. } To this end, dataset distillation algorithms construct synthetic datasets (not necessarily a subset from the original input) such that gradient descent training on the synthetic datapoints results in high predictive accuracy on the real dataset. Cast as a bilevel optimization problem, early methods involve unrolling training computation graph \citep{wang2018dataset} for a few gradient descent steps and randomly sampled weight initializations. 
More sophisticated methods aim to approximate the unrolled computation using kernel methods \citep{KIP1, KIP2, frepo, loo2022efficient,loo2022evolution}, surrogate objectives such gradient matching \citep{zhao2021DC, zhao2021DSA}, trajectory matching \citep{mtt} or implicit gradients \citep{RCIG}.
The kernel-induced points (KIP) algorithm \citep{KIP1, KIP2} is a technique that employs Neural Tangent Kernel (NTK) theory\citep{jacot2018neural,loo2022evolution} to formulate the ensuing loss:
$\mathcal{L}_{KIP} = \frac{1}{2}\|y_t - K_{TS}K_{SS}^{-1}y_S\|^2_2.$ 
This loss signifies the predictive loss of training infinitely wide networks on distilled datapoints $X_S$ with corresponding labels $y_S$, on the original training set and labels $X_T, y_T$, with $K_{\cdot, \cdot}$ being the NTK. Dataset distillation is closely related to the use of inducing points to accelerate Gaussian Processes \citep{fitc, tran2016variational}, for which convergence rates exist, but the existence of such inducing points is not unknown \citep{svgp_convergence}.

\paragraph{From dataset distillation to kernel ridge regression. }
Kernel ridge regression extends the linear machine learning ridge regression model by using a kernel function to map input data into higher-dimensional feature spaces, allowing for more complex non-linear relationships between variables to be captured \cite{murphy2012machine}. Various methods have been proposed to improve and accelerate the training process of kernel ridge regression. Most notably, Random Fourier Features \cite{rahimi2007random} approximates shift-invariant kernel functions by mapping the input data into a lower-dimensional feature space using a randomized cosine transformation. This has been shown to work effectively in practice due to regularizing effects \cite{jacot2020implicit}, as well as providing approximation bounds to the full kernel ridge regression
\cite{sutherland2015error, avron2017random, li2019towards}. Training infinite-width neural networks can be cast as kernel ridge regression with the Neural Tangent Kernel (NTK) \citep{jacot2018neural}, which allows a closed-form solution of the infinite-width neural network's predictions, enabling kernel-based dataset distillation algorithms such as~\citep{KIP1, KIP2, loo2022efficient}.

\paragraph{Goal.} We thus provide the first provable guarantees on the existence and approximation error of a small distilled dataset in the kernel ridge regression settings.  

\section{Background}\label{sec:background} 
We first provide some notation that will be used throughout the paper.  
\paragraph{Notations.} In this paper we let $\mathcal{H}$ be a Hilbert space with $\norm{\cdot}{\mathcal{H}}$ as its norm. For a vector $a \in \REAL^{n}$, we use $\norm{a}{}$ to denote its Euclidean norm, and $a_i$ to denote its $i$th entry for every $i \in [n]$.  
For any positive integer $n$, we use the convention $[n] = \br{1, 2, \cdots, n}$. Let $A \in \REAL^{n \times m}$ be a matrix, then, for every $i\in [n]$ and $j \in [m]$, $A_{i*}$ denotes the $i$th row of $A$, $A_{*j}$ denotes the $j$th column of $A$, and $A_{i,j}$ is the $j$th entry of the $i$th row of $A$. 
Let $B \in \REAL^{n\times n}$, then we denote the trace of $B$ by $Tr(B)$. 
We use $\I_m \in \REAL^{m \times m}$ to denote the identity matrix. 
Finally, vectors are addressed as column vectors unless stated otherwise.
\subsection{Kernel ridge regression} \label{sec:supervised_learning}
Let $\X \in \REAL^{n \times d}$ be a matrix and let $y \in \REAL^n$ be a vector.  Let $k : \REAL^d \times \REAL^d \to [0,\infty) $ be a kernel function,  and let $\mathbf{K}\in \REAL^{n\times n}$ be its corresponding kernel matrix with respect to the rows of $\X$; i.e., $\K_{i,j} = k\term{\X_{i*},\X_{j*}}$ for every $i,j\in[n]$. 
Let $\lambda>0$ be a regularization parameter.
The  goal of kernel ridge regression (KRR) involving $\X, y,k,$ and $\lambda$ is to find 
\begin{equation}
    \begin{split}
    \alphak \in \argmin_{\alpha \in \REAL^n} \frac{1}{n} \norm{y - \K\alpha}{}^2 + \lambda \alpha^T \K\alpha. \label{main:krl_opm}  
    \end{split}
\end{equation}

We use the notation $\fk\colon \REAL^d \to \REAL$ to denote the in-sample prediction by applying the KRR solution obtained on $\X,y$ and $\lambda$ using the kernel $k$, i.e., for every $x \in \REAL^d$, 
\begin{align}
\fk(x) =\sum_{i=1}^n {\alphak}_i k\term{\X_{i*}, x}.
\end{align}

To provide our theoretical guarantees on the size and approximation error for the distilled datasets, the following assumption will be used in our theorem and proofs. 
\begin{assumption}
\label{assumption}
We inherit the same theoretical assumptions used at~\cite{li2021towards} for handling the KRR problem:
\begin{enumerate}[label=(\Roman*)]
\item Let $\mathcal{F}$ be the set of all functions mapping $\REAL^d$ to $\REAL$. Let $f^* \in \mathcal{F}$ be the minimizer of $\sum\limits_{i = 1}^n \abs{y_i - f\term{\mathbf{X}_{i*}}}^2$, subject to the constraint that for every $x\in \REAL^d$ and $y\in \REAL$, $y = f^*(x) + \epsilon$, where $\mathbb{E}(\epsilon) = 0$ and $\text{Var}(\epsilon) = \sigma^2$. Furthermore, we assume that $y$ is bounded, i.e., $\abs{y} \leq y_0$. \label{assumption:1}
\item We assume that $\norm{\fk}{\mathcal{H}} \leq 1$. \label{assumption:2} 
\item For a kernel $k$, denote with $\lambda_1 \geq \cdots \geq\lambda_n$ the eigenvalues of the kernel matrix $\K$. We assume that the regularization parameter satisfies $0 \leq n\lambda \leq \lambda_1$. \label{assumption:3} 
\end{enumerate}
\end{assumption}

\paragraph{The logic behind our assumptions.} First, the idea behind Assumption~\ref{assumption:1} is that the pair $\term{\X, y}$ can be linked through some function that can be from either the same family of kernels that we support (i.e., shift-invariant) or any other kernel function. In the context of neural networks, the intuition behind Assumption~\ref{assumption:1} is that there exists a network from the desired architectures that gives a good approximation for the data. Assumption~\ref{assumption:2} aims to simplify the bounds used throughout the paper as it is a pretty standard assumption, characteristic to the analysis of random Fourier features~\cite{li2019towards,rudi2017generalization}.
Finally, Assumption~\ref{assumption:3} is to prevent underfitting. Specifically speaking, the largest eigenvalue of $\K\term{\mathbf{K}+ n\lambda \I_n}^{-1}$ is $\frac{\lambda_1}{(\lambda_1 + n\lambda)}$. Thus, in the case of $n\lambda > \lambda_1$, the in-sample prediction is dominated by the term $n\lambda$. Throughout the following analysis, we will use the above assumptions. Hence, for the sake of clarity, we will not repeat them, unless problem-specific clarifications are required. 

\paragraph{Connection to Dataset distillation of neural networks.} Since the neural network kernel in the case of infinite width networks describes a Gaussian distribution~\cite{jacot2018neural}, we aim at proving the existence of small sketches (distilled sets) for the input data with respect to the KRR problem with Gaussian kernel.
However, the problem with this approach is that the feature space (in the Gaussian kernel corresponding mapping) is rather intangible or hard to map to, and sketch (distilled set) construction techniques require the representation of these points in the feature space. 

To resolve this problem, we use a randomized approximated feature map, e.g., random Fourier features (RFF), and weighted random Fourier features (Weighted RFF). The dot product between every two mapped vectors in this approximated feature map aims to approximate their Gaussian kernel function~\cite{rahimi2007random}. 
We now restate a result connecting ridge regression in the RFF space (or alternatively weighted RFF), and KRR in the input space.

\begin{theorem}[A result of the proof of Theorem~1 and Corollary~2 of~\cite{li2021towards}]
\label{main:krr_emperical_risk}
Let $\X \in \REAL^{n \times d}$ be an input matrix, $y \in \REAL^n$ be an input label vector, $k : \REAL^d \times \REAL^d \to [0,\infty)$ be a shift-invariant kernel function, and $\K\in \REAL^{n\times n}$, where $\forall i,j \in[n]: \K_{i,j} = k(\X_{i*},\X_{j*})$.
Let $\lambda>0$, and let $d^\lambda_{\K} = Tr\term{\K\term{\mathbf{K}+n\lambda\I_n}^{-1}}$.
Let $s_\phi \in \Omega\term{{d^\lambda_{\K} \log\term{d^\lambda_{\K}}}{}}$ be a positive integer.
Then, there exists a pair $(\phi,\RFFX)$ such that
    (i) $\phi$ is a mapping $\phi : \REAL^d \to \REAL^{s_\phi}$ (which is based on either the weighted RFF function or the RFF function~\cite{li2021towards}), 
    (ii) $\RFFX$ is a matrix  $\RFFX \in \REAL^{n \times s_\phi}$ where for every $i \in [n]$, $\RFFX_{i*} := \phi\term{\mathbf{X}_{i*}}$, and
    (iii) $(\phi,\RFFX)$  satisfies
$$ \frac{1}{n} \sum\limits_{i = 1}^n \abs{y_i - \fkt\term{\RFFX_{i*}}}^2  \leq \frac{1}{n} \sum\limits_{i = 1}^n \abs{y_i - \fk\term{\mathbf{X}_{i*}}}^2 + 4\lambda,$$
where $\fkt: \REAL^{s_\phi} \to \REAL$ such that for every row vector $z \in \REAL^{s_\phi}$, $\fkt(z) = z \term{\RFFX^T\RFFX + \lambda n s_\phi \lambda \mathbf{I}_{s_\phi}}^{-1}\RFFX^T y$. Note that, Table~\ref{tab:squ-wor} gives bounds on $s_\phi$ when $\lambda \propto \frac{1}{\sqrt{n}}$. 

\end{theorem}

\section{Main result: on the existence of small distilled sets} \label{sec:main}
In what follows, we show that for any given matrix $\X \in \REAL^{n \times d}$  and a label vector $y \in \REAL^n$, there exists a matrix $\BS\in \REAL^{\term{s_\phi + 1} \times d}$ and a label vector $y_{\BS} \in \REAL^{s_\phi + 1}$ such that the fitting solution in the RFF space mapping of $\BS$ is identical to that of the fitted solution on the RFF space mapping of $\X$. With such $\BS$ and $y_{\BS}$, we proceed to provide our main result showing that one can construct a solution for KRR in the original space of $\BS$ which provably approximates the quality of the optimal KRR solution involving $\X$ and $y$.  Thus, we obtain bounds on the minimal distilled set size required for computing a robust approximation, as well as bounds on the error for such a distilled set.
 
\begin{theorem}[On the existence of some distilled data]
\label{thm:existence}
Let $\X \in \REAL^{n \times d}$ be a matrix, $y \in \REAL^n$ be a label vector, $k : \REAL^d \times \REAL^d \to [0,\infty) $ be a kernel function, $\Upsilon = \term{0,1} \cup \br{2}$, and let $s_\phi$ be defined as in Theorem~\ref{main:krr_emperical_risk}.   
Then, there exists a matrix $\mathbf{S} \in \REAL^{\term{s_\phi + 1} \times d}$ and a label vector $\y{S}$ such that
\begin{enumerate}[label=(\roman*)]
    \item  the weighted RFF mapping $\rahimi{\BS} \in \REAL^{\term{s_\phi + 1} \times \term{s_\phi}}$ of $\BS$,  satisfies that 
    $$
    \term{\RFFX^T\RFFX + \lambda n s_\phi \lambda \I_{s_\phi}}^{-1}\RFFX^T y = \term{\rahimi{S}^T\rahimi{S} + \lambda n s_\phi \lambda \I_{s_\phi}}^{-1}\rahimi{S}^T \y{S},
    $$ and \label{thm_guarantee:1}
    \item \label{thm_guarantee:2} there exists an in-sample prediction $\fks$ (not necessarily the optimal on $\BS$ and $\y{s}$) satisfying
    \begin{equation}
    \label{eq:bound_avg_dist_to_opt_sol}
    \begin{split}
    \frac{1}{n} \sum\limits_{i = 1}^n \abs{\fk\term{\X_{i*}} - \fks \term{\X_{i*}}}^2 \leq &\min_{\tau \in \Upsilon}\left( 2\maxterm{\tau}{\frac{4}{\tau^2}} + \right. \\
    &\left. \,\,\,\,\, 2\minterm{1+\tau}{\frac{4\term{1+\tau}}{3\tau}} \right)\lambda,
    \end{split}
    \end{equation}
     and
   \begin{equation}
   \begin{split}  \label{eq:bound_avg_dist_to_y}
    \frac{1}{n} \sum\limits_{i = 1}^n \abs{y_i - \fks \term{\X_{i*}}}^2  &\leq \min_{\tau \in \Upsilon} \frac{\minterm{1+\tau}{\frac{4\term{1+\tau}}{3\tau}}}{n} \sum\limits_{i = 1}^n \abs{y_i - \fk\term{\X_{i*}}}^2 \\ 
    & \,\,\,+ \term{4\minterm{1+\tau}{\frac{4\term{1+\tau}}{3\tau}} + 2\maxterm{\tau}{\frac{4}{\tau^2}}}\lambda .
   \end{split}
   \end{equation}
  
\end{enumerate}
\end{theorem}

\begin{proof}
Let $\BS$ be any matrix in $\REAL^{\term{s_\phi + 1} \times d}$ and let $\rahimi{\BS}$ be the weighted RFF mapping of $\BS$.

\newcommand{\pinv}[1]{\term{#1}^{\dag}}
\paragraph{Proof of~\ref{thm_guarantee:1}.} To ensure~\ref{thm_guarantee:1}, we need to find a corresponding proper $y_\BS$.  We observe that 
\[
\term{\rahimi{\BS}^T\rahimi{S} + \lambda n s_\phi \lambda \I_{s_\phi}}\term{\RFFX^T\RFFX + \lambda n s_\phi \lambda \I_{s_\phi}}^{-1}\RFFX^T y = \rahimi{\BS}^T \y{\BS}
\]
Let $b = \term{\rahimi{\BS}^T\rahimi{\BS} + \lambda n s_\phi \lambda \I_{s_\phi}}\term{\RFFX^T\RFFX + \lambda n s_\phi \lambda \I_{s_\phi}}^{-1}\RFFX^T y$, be the left-hand side term above. $b$ is a vector of dimension $s_\phi$. Hence we need to solve $b = \rahimi{\BS}^T \y{\BS}$ for  $\y{\BS}$.  Since it is a linear system with $s_\phi +1$ variables and $s_\phi$ equations, we get that the solution is $\y{\BS} = \pinv{\rahimi{\BS}^T}b$, where  $\pinv{\cdot}$ denotes the pseudo-inverse of the given matrix.

\paragraph{Proof of~\ref{thm_guarantee:2}.} Inspired by~\cite{loo2022efficient} and~\cite{nguyen2020dataset}, the goal is to find a set of instances that their in-sample prediction with respect to the input data ($\X$ in our context) would lead to an approximation towards the solution that one would achieve if the KRR was used only with the input data. To that end, we introduce the following Lemm.
\begin{lemma}[Restatement of Lemma 6~\cite{li2021towards}]\label{func_appx_opm} Under Assumption~\ref{assumption} and the definitions in Theorem~\ref{main:krr_emperical_risk}, for every $f \in \mathcal{H}$ with $\norm{f}{\mathcal{H}} \leq 1 $, with constant probability, it holds that
\begin{equation*}
\inf\limits_{\substack{\sqrt{s_\phi}\norm{\beta}{} \leq \sqrt{2} \\ \beta \in \REAL^{s_\phi}}} \, \, \sum\limits_{i=1}^n \frac{1}{n}\abs{f\term{\mathbf{X}_{i*}}- \RFFX_{i*} \beta}^2 \leq 2\lambda.
\end{equation*}
\end{lemma}

Note that Lemma~\ref{func_appx_opm} shows that for every in-sample prediction function with respect to $\X$, there exists a query $\beta \in \REAL^{s_\phi}$ in the RFF space of that input data such that the distance between the in-prediction sample function in the input space and the in-sample prediction in the RFF space is at $2\lambda$. 
Furthermore, at~\cite{li2021towards} it was shown that $\beta$ is defined as  
$
\beta = \frac{1}{s_\phi} \RFFX^T\term{\RFFX\RFFX^T + n\lambda\I_{s_\phi}}^{-1}\mathbf{f}[\X],
$ where $\mathbf{f}[\X]_i = f\term{\X_{i*}}$ for every $i \in [n]$. 

We thus set out to find an in-sample prediction function that is defined over $\BS$ such that by its infimum by Lemma~\ref{func_appx_opm} would be the same solution $\beta$ that the ridge regression on $\RFFX$ attains with respect to the $y$. Specifically speaking, we want to find an in-sample prediction  $\fks\term{\cdot}$ such that
\begin{equation}
\label{eq:beta}
\beta = \frac{1}{s_\phi} \RFFX^T\term{\frac{1}{s_\phi} \RFFX\RFFX^T + n\lambda \I_{s_\phi}}^{-1} \mathbf{f}_{\BS}[\X],    
\end{equation}
where 
\begin{enumerate}[label = (\roman*)]
    \item $\mathbf{f}_{\BS}[\X] \in \REAL^n$ such that for every $i \in [n]$, $\mathbf{f}_{\BS}[\X]_{i} = \fks\term{\mathbf{X}_{i*}}$, and 
    \item $\fks\term{\cdot} = \sum\limits_{i=1}^{s_\phi + 1} \alpha_i k\term{\mathbf{S}_{i*}, \cdot}$ such that $\alpha \in \REAL^{s_\phi + 1}$. \label{alphause}
\end{enumerate}

Hence we need to find an in-sample prediction function $\fks$ satisfying~\ref{eq:beta}. Now, notice that $\beta \in \REAL^{s_\phi}$, $\mathbf{f}_{\mathbf{S}}[\mathbf{X}] \in \REAL^n$ and $\RFFX^T\term{\frac{1}{s_\phi} \RFFX\RFFX^T + n\lambda \I_{n}}^{-1} \in \REAL^{{s_\phi} \times n}$.
Due to the fact that we aim to find $\fks$, such a task boils down to finding $\alpha \in \REAL^{s_\phi+1}$ which defines $\fks$ as in~\ref{alphause}. The above problem can be reduced to a system of linear equations where the number of equalities is $s_\phi$, while the number of variables is $s_\phi +1$.

To do so, we denote $\frac{1}{s_\phi} \RFFX^T\term{\frac{1}{s_\phi} \RFFX\RFFX^T + n\lambda \I_{n}}^{-1}$ by $\hat{\A}$, and observe that we aim to solve 
\[
\beta = \hat{\A} f^\lambda_{\BS}[\X] = \hat{\A} \begin{bmatrix}
\sum\limits_{i = 1}^{s_\phi + 1} \alpha_i k\term{\BS_{i*}, \X_{1*}}\\
\sum\limits_{i = 1}^{s_\phi + 1} \alpha_i k\term{\BS_{i*}, \X_{2*}}\\
\vdots\\
\sum\limits_{i = 1}^{s_\phi + 1} \alpha_i k\term{\BS_{i*}, \X_{n*}}
\end{bmatrix}.
\]

We now show that every entry $b_j$ ($j \in [s_\phi + 1]$) in $\beta$ can be rewritten as inner products between another pair of vectors in $\REAL^{s_\phi + 1}$ instead of the inner product between two vectors in $\REAL^n$. Formally, for every $j \in [s_\phi + 1]$, it holds that

\begin{equation*}
\beta_j = \hat{\A}_{j*} \begin{bmatrix}
\sum\limits_{i = 1}^{s_\phi + 1} \alpha_i k\term{\BS_{i*}, \X_{1*}}\\
\sum\limits_{i = 1}^{s_\phi + 1} \alpha_i k\term{\BS_{i*}, \X_{2*}}\\
\vdots\\
\sum\limits_{i = 1}^{s_\phi + 1} \alpha_i k\term{\BS_{i*}, \X_{n*}}
\end{bmatrix} = \begin{bmatrix}
\sum\limits_{t=1}^{n} \hat{\A}_{j,t} k\term{\BS_{1*}, \X_{t*}}, \cdots,
\sum\limits_{t=1}^{n} \hat{\A}_{j,t} k\term{\BS_{\term{s_\phi + 1}*}, \X_{t*}}
\end{bmatrix} 
\begin{bmatrix} \alpha_1\\ \vdots \\ \alpha_{s_\phi+1} \end{bmatrix}.
\end{equation*}

Thus, for every $j \in [s_\phi + 1]$, define $$\A_{j*} =\begin{bmatrix}
\sum\limits_{t=1}^{n} \hat{\A}_{j,t} k\term{\BS_{1*}, \X_{t*}},
\cdots,
\sum\limits_{t=1}^{n} \hat{\A}_{j,t} k\term{\BS_{\term{s_\phi + 1}*}, \X_{t*}}
\end{bmatrix} \in \REAL^{s_\phi + 1}.$$ The right-hand side of~\eqref{eq:beta} can reformulated  as 
\begin{equation}
\label{eq:reformulate_beta}
\frac{1}{s_\phi} \RFFX^T\term{\frac{1}{s_\phi} \RFFX\RFFX^T + n\lambda \I_{n}}^{-1} \mathbf{f}_{\BS}[\X] = \A \alpha,
\end{equation}
where now we only need to solve 
$\beta = \A \alpha. $ 
Such a linear system of equations might have an infinite set of solutions due to the fact that we have $s_\phi + 1$ variables (the length of $\alpha$) and exactly $s_\phi$ equations. For simplicity, a solution to the above equality would be 
$
\alpha := \pinv{\A}\beta.
$ 

To proceed in proving~\ref{thm_guarantee:2} with all of the above ingredients, we utilize the following tool.

\begin{lemma}[Special case of Definition~6.1 from~\cite{braverman2016new}]
\label{lem:weak_tri_ineq}
Let $X$ be a set, and let $\term{X,\norm{\cdot}{}^2}$ be a $2$-metric space i.e., for every $x,y,z \in X$, 
$
\norm{x-y}{}^2  \leq 2 \term{\norm{x-z}{}^2 + \norm{y-z}{}^2}.
$
Then, for every $\eps \in (0,1)$, and $x,y,z \in X$,
\begin{equation}
\label{eq:weak_tri_ineq_eps}
\term{1-\eps} \norm{y-z}{}^2 - \frac{4}{\eps^2} \norm{x-z}{}^2\leq \norm{x-y}{}^2 \leq \frac{4}{\eps^2} \norm{x-z}{}^2 + \term{1+\eps} \norm{y-z}{}^2.
\end{equation}
\end{lemma}

We note that Lemma~\ref{lem:weak_tri_ineq} implies that $x,y,z \in \REAL^d$
\begin{equation}
\label{eq:modified_weak_tri_ineq}
\norm{x-y}{}^2 \leq \min_{\tau \in \Upsilon} \maxterm{\tau}{\frac{4}{\tau^2}} \norm{x-z}{}^2 + \minterm{1+\tau}{\frac{4\term{1+\tau}}{3\tau}}\norm{y-z}{}^2.  
\end{equation}
where for $\tau = 2$ we get the inequality associated with the property of $2$-metric, and for any $\tau \in (0,1)$, we obtain the inequality~\eqref{eq:weak_tri_ineq_eps}.

We thus observe that  
\begin{equation*}
\begin{split}
&\frac{1}{n} \sum\limits_{i = 1}^n \abs{\fk\term{\X_{i*}} - \fks \term{\X_{i*}}}^2\\
&\quad = \frac{1}{n} \sum\limits_{i = 1}^n \abs{\fk\term{\X_{i*}} - \fkt \term{\RFFX_{i*}} + \fkt \term{\RFFX_{i*}} -  \fks \term{\X_{i*}}}^2\\
&\quad \leq \min_{\tau \in \Upsilon} \frac{\maxterm{\tau}{\frac{4}{\tau^2}}}{n} \sum\limits_{i = 1}^n \abs{\fk\term{\X_{i*}} -\fkt\term{\RFFX_{i*}}}^2 +\\ &\, \, \, \, \quad\quad \frac{\minterm{1+\tau}{\frac{4\term{1+\tau}}{3\tau}}}{n}\sum\limits_{i=1}^n \abs{\fkt \term{\RFFX_{i*}} -  \fks \term{\X_{i*}}}^2\\
&\quad \leq \min_{\tau \in \Upsilon} 2\maxterm{\tau}{\frac{4}{\tau^2}}\lambda + 2\minterm{1+\tau}{\frac{4\term{1+\tau}}{3\tau}}\lambda \\
&= \min_{\tau \in \Upsilon}\term{2\maxterm{\tau}{\frac{4}{\tau^2}} + 2\minterm{1+\tau}{\frac{4\term{1+\tau}}{3\tau}}}\lambda ,
\end{split}
\end{equation*}
where the first equality holds by adding and subtracting the same term, the first inequality holds by Lemma~\ref{lem:weak_tri_ineq}, and the second inequality holds by combining the way $\fks$ was defined and Theorem~\ref{main:krr_emperical_risk}.

Finally, to conclude the proof of Theorem~\ref{thm:existence}, we derive~\ref{eq:bound_avg_dist_to_y}
\begin{equation}
\label{eq:using_weak_triangle_ineq}
\begin{split}
&\frac{1}{n} \sum\limits_{i = 1}^n \abs{y_i - \fks \term{\X_{i*}}}^2 = \frac{1}{n} \sum\limits_{i = 1}^n \abs{y_i - \fkt \term{\RFFX_{i*}} + \fkt \term{\RFFX_{i*}} - \fks \term{\X_{i*}}}^2\\
&\leq \min_{\tau \in \Upsilon} \frac{\minterm{1+\tau}{\frac{4\term{1+\tau}}{3\tau}}}{n} \sum\limits_{i = 1}^n \abs{y_i - \fkt \term{\RFFX_{i*}}}^2 \\
&\,\,\,\,\,\, + \frac{\maxterm{\tau}{\frac{4}{\tau^2}}}{n} \sum\limits_{i=1}^n \abs{\fkt \term{\RFFX_{i*}} - \fks \term{\X_{i*}}}^2 \\
&\leq  \min_{\tau \in \Upsilon} \frac{\minterm{1+\tau}{\frac{4\term{1+\tau}}{3\tau}}}{n}  \sum\limits_{i = 1}^n \abs{y_i - \fkt \term{\RFFX_{i*}}}^2 + 2\maxterm{\tau}{\frac{4}{\tau^2}} \lambda\\
&\leq  \min_{\tau \in \Upsilon} \frac{\minterm{1+\tau}{\frac{4\term{1+\tau}}{3\tau}}}{n} \sum\limits_{i = 1}^n \abs{y_i - \fk\term{\X_{i*}}}^2 \\ &\,\,\,\,\,\, \,\,\,\,\,\, + \term{4\minterm{1+\tau}{\frac{4\term{1+\tau}}{3\tau}} + 2\maxterm{\tau}{\frac{4}{\tau^2}}}\lambda,
\end{split}
\end{equation}
where the equality holds by adding and subtracting the same term, the first inequality holds by~\eqref{eq:modified_weak_tri_ineq}, and the second inequality follows as a result of the way $f^\lambda_{\mathbf{S}}$ was constructed and the fact that $\beta$ is its infimum based on Lemma~\ref{func_appx_opm}, and the last inequality holds by Theorem~\ref{main:krr_emperical_risk}.
\end{proof}

To simplify the bounds stated at Theorem~\ref{thm:existence}, we provide the following remark.
\begin{remark}\label{rem:main}
By fixing $\tau := 2$, the bounds in Theorem~\ref{thm:existence} become
\[
\frac{1}{n} \sum\limits_{i = 1}^n \abs{\fk\term{\X_{i*}} - \fks \term{\X_{i*}}}^2 \leq 8\lambda,
\]
and 
\[
\frac{1}{n} \sum\limits_{i = 1}^n \abs{y_i - \fks \term{\X_{i*}}}^2  \leq \frac{2}{n} \sum\limits_{i = 1}^n \abs{y_i - \fk\term{\X_{i*}}}^2 + 12\lambda.
\]

As for fixing $\tau := \eps \in (0,1)$, we obtain that 
\[
\frac{1}{n} \sum\limits_{i = 1}^n \abs{y_i - \fks \term{\X_{i*}}}^2  \leq \frac{1+\eps}{n} \sum\limits_{i = 1}^n \abs{y_i - \fk\term{\X_{i*}}}^2 + \term{4\term{1+\eps} + \frac{8}{\eps^2}}\lambda.
\]
\end{remark}

\section{Experimental Study}
\label{sec:experiments}
To validate our theoretical bounds, we performed distillation on three datasets: two synthetic datasets consisted of data generated from a Gaussian Random Field \cref{sec:exp_grf}, and classification of two clusters \cref{sec:exp_two_clusters}, and one real dataset of MNIST binary 0 vs. 1 classification \cref{sec:exp_mnist}. Full experimental details for all experiments are available in the appendix.
\subsection{2d Gaussian Random Fields}
\label{sec:exp_grf}

\begin{figure}[h]
\begin{center}
\subfloat[][]{\includegraphics[height = 1.4in]{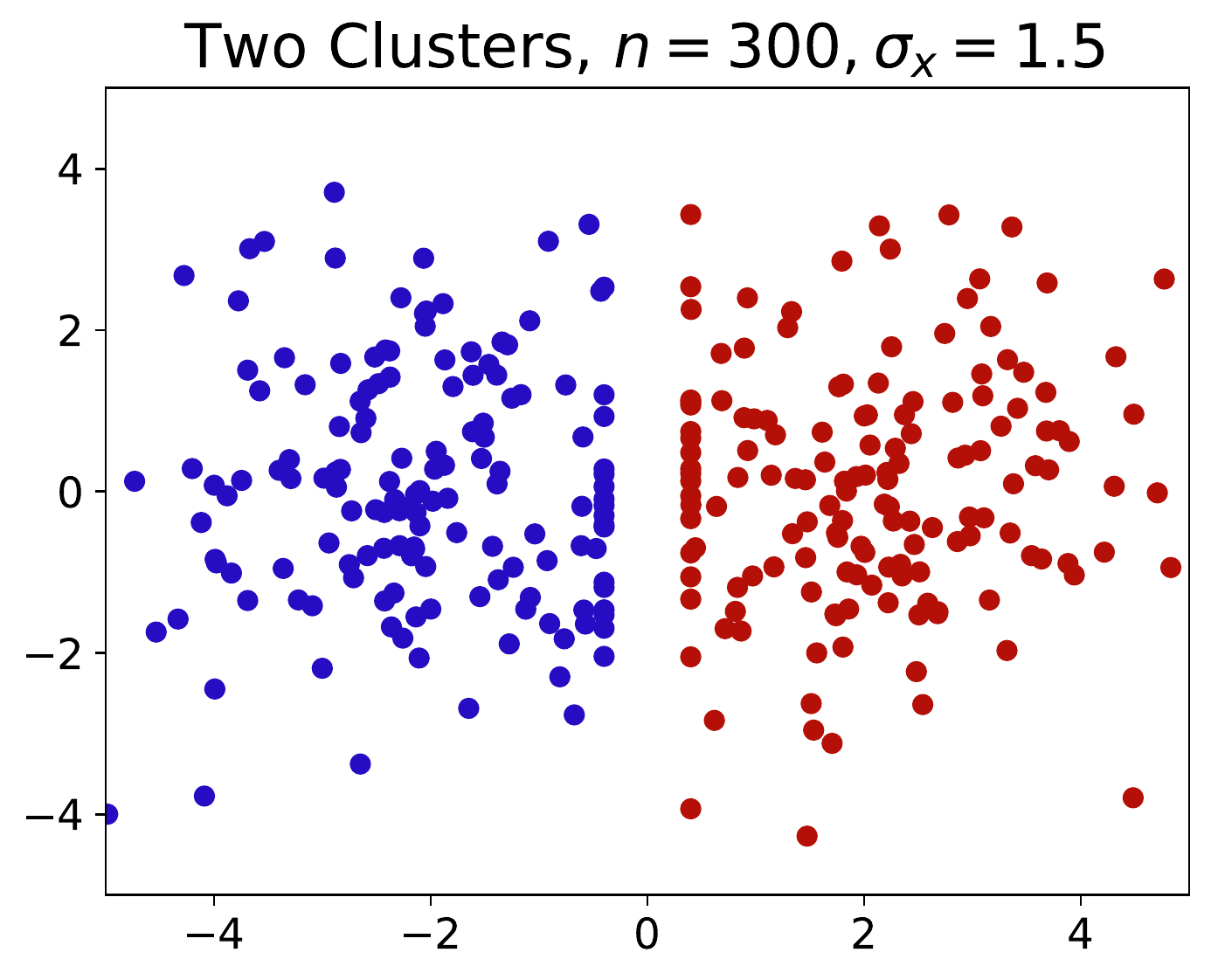}\label{fig:two_clusters_vis}}
\subfloat[][]{\includegraphics[height = 1.4in]{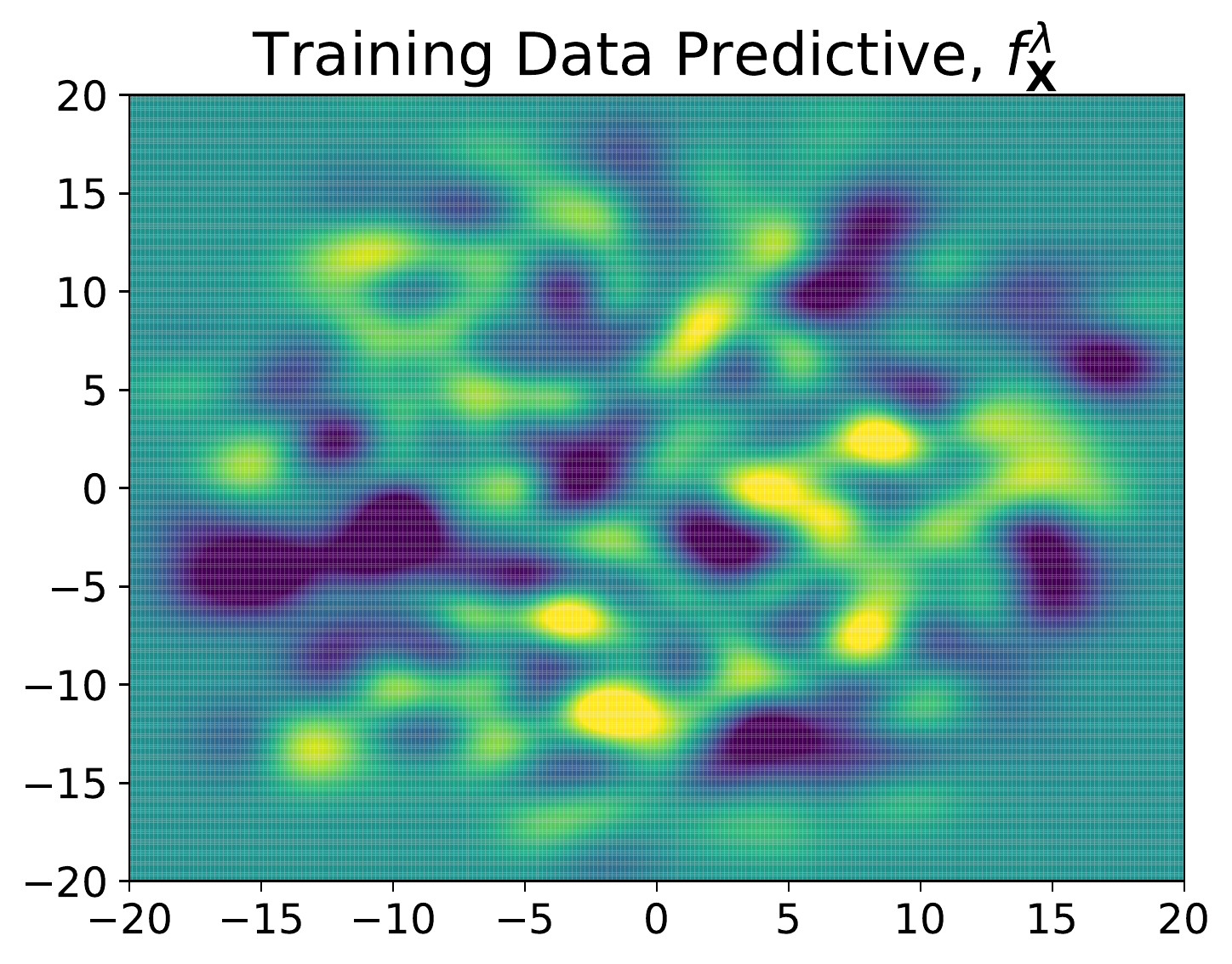}\label{fig:grf_real_predictive}}
\subfloat[][]{\includegraphics[height = 1.4in]{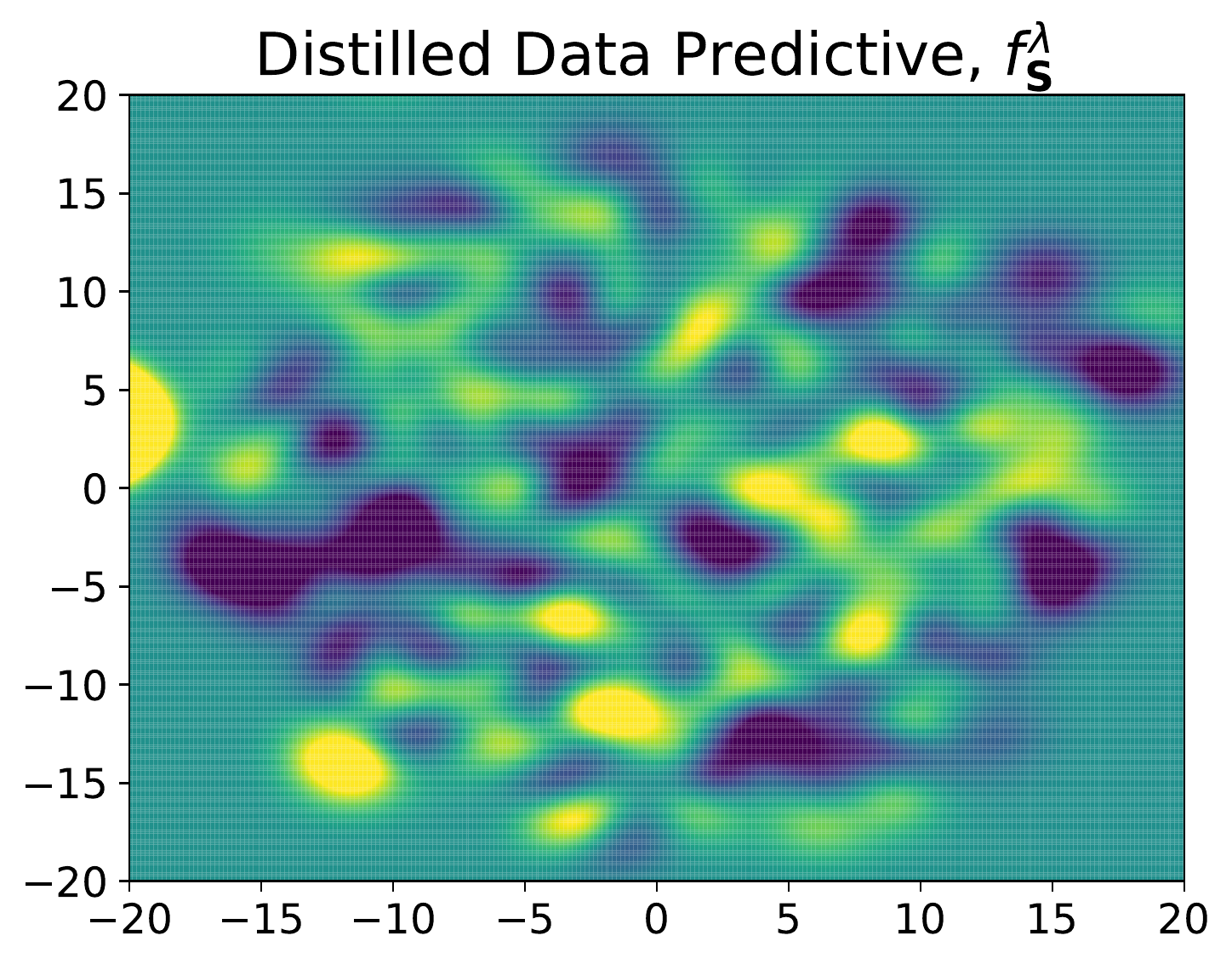}\label{fig:grf_distilled_predictive}}
\vskip -0.1in
\caption{(a) visualizes the two clusters dataset in \cref{sec:exp_two_clusters} with $n = 300$ and $\sigma_x = 1.5$. (b) and (c) visualizing the KRR predictive functions generated by the original dataset (b) and the distilled dataset (c) for the Gaussian Random Field experiment in \cref{sec:exp_grf} for $\sigma_x = 5.0$. The distilled dataset is able to capture all the nuances of the original dataset with a fraction of the datapoints.}
\label{fig:visualization_banner}
\end{center}
\vspace{-5mm}
\end{figure}

\begin{figure}[h]
\begin{center}
\includegraphics[width = 1.0\linewidth]{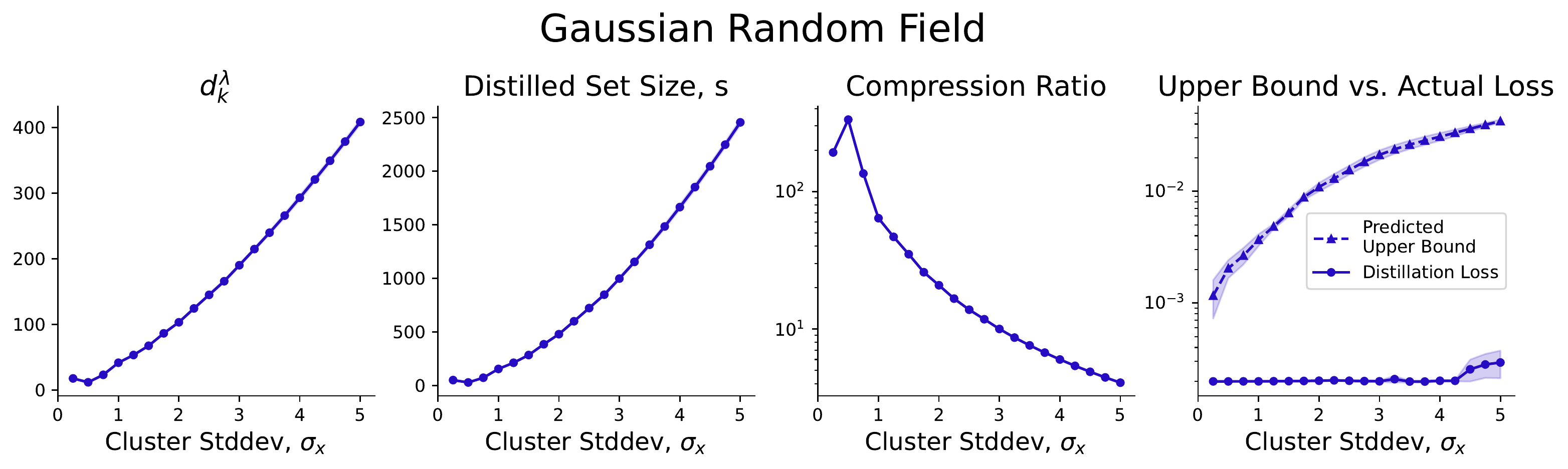}
\vskip -0.1in
\caption{Distillation results for synthetic data generated by a Gaussian Random Field (n = 3)}
\label{fig:grf}
\end{center}
\vspace{-5mm}
\end{figure}

We first test our bounds by distilling data generated from the Gaussian Process prior induced by a kernel, $k$ on 2d data. We use a squared exponential kernel with lengthscale parameter $l = 1.5$: $k(x, x') = e^{-\frac{||x-x'||^2_2}{2l^2}}$. For $\mathbf{X}$, we sample $n = 10^5$ datapoints from $\mathcal{N}(0, \sigma^2_{x})$, with $\sigma_x \in [0.25, 5.0]$. We then sample $y \sim \mathcal{N}(0, K_{XX} + \sigma_y^2 I_n)$, $\sigma_y = 0.01$. We fix $\lambda = 10^{-5}$ and distill down to $s = d_k^\lambda \log d_k^\lambda$. The resulting values of $d_k^\lambda$, $s$, and compression ratios are plotted in \cref{fig:grf}. We additionally plot the predicted upper bound given by Remark~\ref{rem:main} and the actual distillation loss. Our predicted upper bound accurately bounds the actual distillation loss. To better visualize how distillation affects the resulting KRR prediction, we show the KRR predictive function $f^\lambda_{\mathbf{X}}$ and the distilled predictive $f^\lambda_{\mathbf{S}}$ for $\sigma_x = 5.0$ in \cref{fig:grf_real_predictive} and \cref{fig:grf_distilled_predictive}.

\subsection{Two Gaussian Clusters Classification}
\label{sec:exp_two_clusters}

\begin{figure}[h]
\begin{center}
\includegraphics[width = 1.0\linewidth]{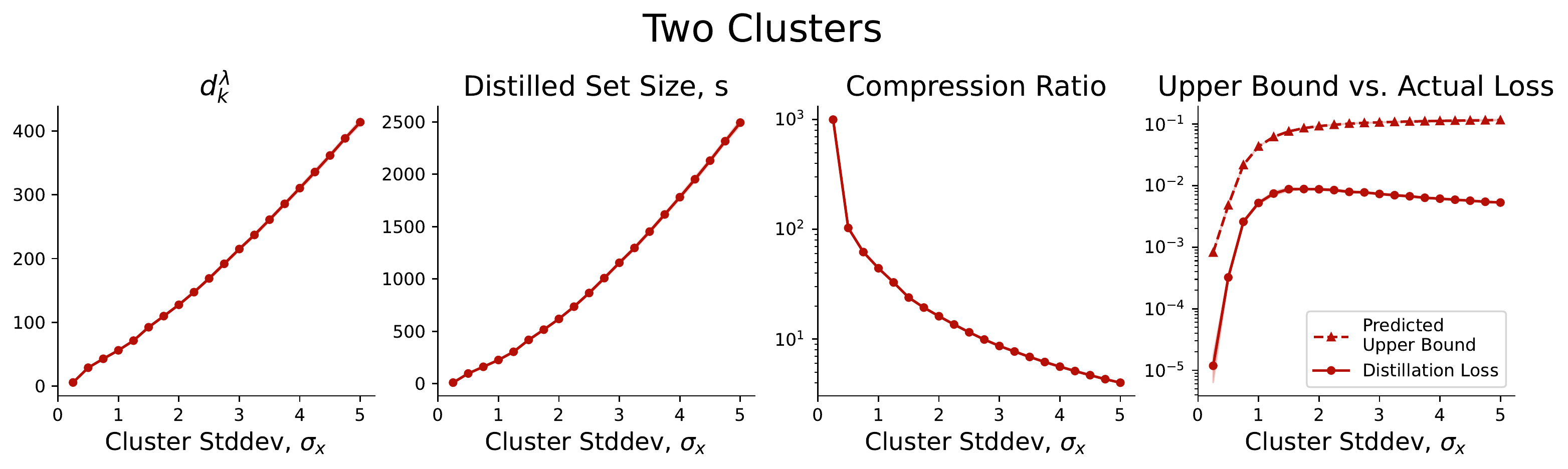}
\vskip -0.1in
\caption{Distillation results for synthetic data of two Gaussian clusters (n = 3)}
\label{fig:two_clusters}
\end{center}
\vspace{-5mm}
\end{figure}

Our second synthetic dataset is one consisting of two Gaussian clusters centered at $(-2, 0)$ and $(2, 0)$, with labels $-1$ and $+1$, respectively. Each cluster contains 5000 datapoints so that $n = 10^5$. Each cluster as standard deviation $\sigma_x \in [0.25, 5.0]$. Additionally, two allow the dataset to be easily classified, we clip the $x$ coordinates of clusters 1 and clusters 2 to not exceed/drop below $-0.4$ and $0.4$, for the two clusters, respectively. This results in a margin between the two classes. We visualize the dataset for $n = 300$ and $\sigma = 1.5$ in \cref{fig:two_clusters_vis}. We use the same squared exponential kernel as in \cref{sec:exp_grf} with $l = 1.5$, fix $\lambda = 10^{-5}$ and distill with the same protocol as in \cref{sec:exp_grf}. We likewise plot $d_k^\lambda$, $s$, and compression ratios and distillation losses in \cref{fig:two_clusters}, again with our bound accurately containing the true distillation loss.

\subsection{MNIST Binary Classification}
\label{sec:exp_mnist}

\begin{figure}[h]
\begin{center}
\includegraphics[width = 1.0\linewidth]{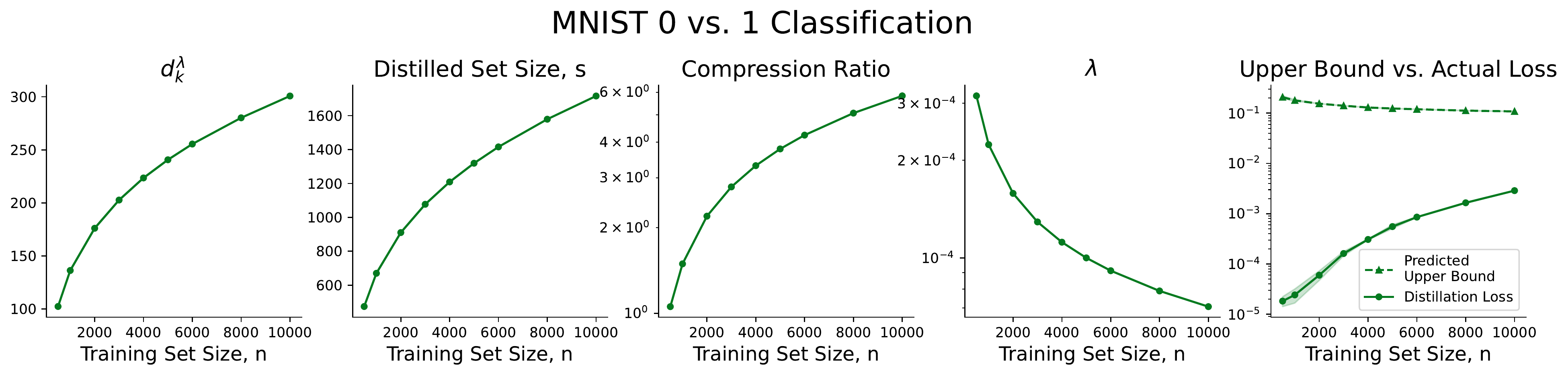}
\vskip -0.1in
\caption{Distillation results for MNIST binary 0 vs. 1 classification (n = 3)}
\label{fig:mnist}
\end{center}
\vspace{-5mm}
\end{figure}

For our final dataset, we consider binary classification on MNIST 0 and 1 digits, with labels $-1$ and $+1$, respectively. We use the same squared-exponential kernel with $l = 13.9$, which was chosen to maximize the marginal-log-likelihood, treating the problem as Gaussian Process regression. We vary $n \in [500, 10000]$, with an equal class split, and perform the same distillation protocol as in \cref{sec:exp_grf}. Here, we additionally scale $\lambda \propto \frac{1}{\sqrt{n}}$ such that $\lambda = 10^{-4}$ when $n = 5000$. Distilling yields \cref{fig:mnist}, showing that our bounds can accurately predict distillation losses for real-world datasets.

\begin{table}[t]
	\centering
 \caption{Table~1 from~\cite{li2019towards}. The trade-off in the worst case for the squared error loss.}
	\begin{tabular}{l|c|c}
		\hline
		
		\hline
		
		\textsc{sampling scheme}&\textsc{spectrum} & \textsc{number of features}\\\cline{1-3}
		
		\multirow{4}{*}{\textsc{weighted rff}} & finite rank & $s_\phi \in \Omega(1)$\\\cline{2-3}
		
		&$\lambda_i \propto A^{i}$ & $s_\phi \in \Omega (\log n \cdot \log \log n)$  	\\\cline{2-3}
		
		&$\lambda_i \propto i^{-2t} $ ($t\geq 1$) & $s_\phi \in \Omega (n^{1/2t} \cdot \log n)$  	\\\cline{2-3}
		
		&$\lambda_i \propto i^{-1}$ & $s_\phi \in \Omega (\sqrt{n} \cdot \log n)$  	\\\cline{2-3}
		
		\hline
		
		\hline
		\multirow{4}{*}{\textsc{plain rff}} & finite rank & $s_\phi \in \Omega(\sqrt{n})$\\\cline{2-3}
		
		&$\lambda_i \propto A^{i}$ & $s_\phi \in \Omega (\sqrt{n} \cdot \log \log n)$  	\\\cline{2-3}
		
		&$\lambda_i \propto i^{-2t} $ ($t\geq 1$) & $s_\phi \in \Omega (\sqrt{n} \cdot \log n)$	\\\cline{2-3}
		
		&$\lambda_i \propto  i^{-1}$ & $s_\phi \in \Omega (\sqrt{n} \cdot \log n)$  	\\\cline{2-3}
		
		\hline
		
		\hline
	\end{tabular}
	\label{tab:squ-wor}
\end{table}

\section{Conclusion}
In this study, we adopt a theoretical perspective to provide bounds on the (sufficient) size and approximation error of distilled datasets.  By leveraging the concept of random Fourier features (RFF), we prove the existence of small distilled datasets and we bound their corresponding excess risk when using shift-invariant kernels. Our findings indicate that the size of the guaranteed distilled data is a function of the "number of effective degrees of freedom," which relies on factors like the kernel, the number of points, and the chosen regularization parameter, $\lambda$, which also controls the excess risk.

In particular, we demonstrate the existence of a small subset of instances within the original input space, where the solution in the RFF space coincides with the solution found using the input data in the RFF space. Subsequently, we show that this distilled subset of instances can be utilized to generate a KRR solution that approximates the KRR solution obtained from the complete input data. To validate these findings, we conducted empirical examinations on both synthetic and real-world datasets supporting our claim.

While this study provides a vital first step in understanding the theoretical limitations of dataset distillation, the proposed bounds are not tight, as seen by the gap between the theoretical upper bound and the empirical distillation loss in \cref{sec:experiments}. Future work could look at closing this gap, as well as better understanding the tradeoff between distillation size and relative error.

\section*{acknowledgements}
This research has been funded in part by the Office of Naval Research Grant Number Grant N00014-18-1-2830, DSTA Singapore, and the J. P. Morgan AI Research program.

\bibliographystyle{alpha}
\bibliography{main}
\newpage
\appendix

\section{Experiment Details}
All experiments unless otherwise stated present the average/standard deviations of $n = 3$ runs. Each run consists of a random subset of MNIST 0/1 digits for MNIST binary classification, or random positions of sampled datapoints for synthetic data, and different samples from the GP for the Gaussian Random Field experiment. Distilled datasets are initialized as subsets of the original training data. We distill for 20000 iterations with Adam optimizer with a learning rate of $0.002$ optimizing both images/data positions and labels. We use full batch gradient descent for the synthetic datasets and a maximum batch size of 2000 for the MNIST experiment. For the MNIST experiment we found that particularly for larger values of $n$, with minibatch training, we could obtain lower distillation losses by optimizing for longer, so the closing of the gap between the upper bound and experiment values in \cref{fig:mnist} may be misleading: longer optimization could bring the actual distillation loss lower.

To ensure that assumption \ref{assumption:2} is fulfilled, we scale the labels such that $\norm{\fk}{\mathcal{H}} = 1$. For example, if we are working with MNIST binary classification, with labels $\{+1, -1\}$, we first compute $\norm{\fk}{\mathcal{H}} = r$ using $\{+1, -1\}$ labels, then rescale the labels by $1/r$ so that the labels are $\{+\frac{1}{r}, -\frac{1}{r}\}$. Suppose this results in some upper bound $\mathcal{L}_U$ and some real distillation loss $\mathcal{L}_R$. For the corresponding plots in \cref{fig:grf,fig:two_clusters,fig:mnist}, we plot $r^2\mathcal{L}_U$ and $r^2\mathcal{L}_R$. We do this because the $r$ values for different parameters (such as $n$ or $\sigma_x$) could be different, and scaling for the plots allows the values to be comparable.

In the figures for the upper bounds on the distillation loss we plot the smallest value of the upper bounds in \cref{rem:main}.

\end{document}